\documentclass{article}

\usepackage{microtype}
\usepackage{graphicx}

\usepackage[accepted]{icml2020}

\icmltitlerunning{Feature Noise Induces Loss Discrepancy Across Groups}

\usepackage[utf8]{inputenc} 
\usepackage{hyperref}       
\usepackage{booktabs}       
\usepackage{amsfonts}       
\usepackage{nicefrac}       
\usepackage{microtype} 
\usepackage{amsmath, amsthm, thm-restate}
\usepackage{graphicx}
\usepackage{makecell}
\usepackage{sidecap}
\usepackage{multirow}
\usepackage{xcolor}
\usepackage{subcaption}
\usepackage{enumitem}
\usepackage{url}
\usepackage{tikz}
\usepackage{colortbl}
\usetikzlibrary{arrows, arrows.meta, positioning}

\newcommand\sN{\ensuremath{\mathcal{N}}}
\newcommand\sO{\ensuremath{\mathcal{O}}}

\newcommand\BP{\ensuremath{\mathbb{P}}}

\newcommand\BR{\ensuremath{\mathbb{R}}}

\newcommand{\var}{\text{Var}} 
\newcommand{\cov}{\text{Cov}} 
\newcommand\p[1]{\ensuremath{\left( #1 \right)}} 
\newcommand\pb[1]{\ensuremath{\left[ #1 \right]}} 

\newcommand\half{\ensuremath{\frac{1}{2}}}
\newcommand\eqdef{\ensuremath{\stackrel{\rm def}{=}}} 
\newcommand\refeqn[1]{(\ref{eqn:#1})}

\newcommand\refsec[1]{Section~\ref{sec:#1}}

\newcommand\reffig[1]{Figure~\ref{fig:#1}}

\newcommand\reftab[1]{Table~\ref{tab:#1}}

\newcommand\refprop[1]{Proposition~\ref{prop:#1}}
\newcommand\refdef[1]{Definition~\ref{def:#1}}

\ifthenelse{\isundefined{\definition}}{\newtheorem{definition}{Definition}}{}
\ifthenelse{\isundefined{\assumption}}{\newtheorem{assumption}{Assumption}}{}
\ifthenelse{\isundefined{\hypothesis}}{\newtheorem{hypothesis}{Hypothesis}}{}
\ifthenelse{\isundefined{\proposition}}{\newtheorem{proposition}{Proposition}}{}
\ifthenelse{\isundefined{\theorem}}{\newtheorem{theorem}{Theorem}}{}
\ifthenelse{\isundefined{\lemma}}{\newtheorem{lemma}{Lemma}}{}
\ifthenelse{\isundefined{\corollary}}{\newtheorem{corollary}{Corollary}}{}
\ifthenelse{\isundefined{\alg}}{\newtheorem{alg}{Algorithm}}{}
\ifthenelse{\isundefined{\example}}{\newtheorem{example}{Example}}{}
\newcommand{\E}{\ensuremath{\mathbb{E}}} 

\newcommand\pl[1]{}

\newcommand{\pr}{\ensuremath{\mathbb{P}}}
\newcommand{\pab}[1]{\ensuremath{\left|#1\right|}}

\newcommand{\sld}{\text{SLD}}
\newcommand{\cld}{\text{CLD}}
\newcommand{\ob}{o}

\newcommand{\obg}{o_{+\text{g}}}
\newcommand{\obnog}{o_{-\text{g}}}

\newcommand{\lr}{\ensuremath{\ell_\text{res}}}
\newcommand{\ls}{\ensuremath{\ell_\text{sq}}}

\newcommand{\betg}{\ensuremath{\hat \beta_\text{g}}}

\definecolor{Green}{rgb}{1, 1, 1} 
\definecolor{Red}{rgb}{1, 1, 1} 

\definecolor{Green1}{rgb}{0.9, 0.9, 1} 
\definecolor{Red1}{rgb}{0.9, 0.9, 1} 

\definecolor{Green2}{rgb}{0.95, 0.95, 1} 
\definecolor{Red2}{rgb}{0.95, 0.95, 1} 

\begin{document}
	
	\twocolumn[
	\icmltitle{Feature Noise Induces Loss Discrepancy Across Groups}
	
	
	
	\icmlsetsymbol{equal}{*}
	
	\begin{icmlauthorlist}
		\icmlauthor{Fereshte Khani}{to}
		\icmlauthor{Percy Liang}{to}
	\end{icmlauthorlist}
	
	\icmlaffiliation{to}{Department of Computer Scinece, Stanford University}	
	\icmlcorrespondingauthor{Fereshte Khani}{fereshte@stanford.edu}
	
	\icmlkeywords{Machine Learning, ICML}
	
	\vskip 0.3in
	]
	
	
	
	\printAffiliationsAndNotice{} 

	\begin{abstract}
	The performance of standard learning procedures has been observed to differ widely across groups. 
	Recent studies usually attribute this loss discrepancy to an information deficiency for one group (e.g., one group has less data). 
	In this work, we point to a more subtle source of loss discrepancy---feature noise. 
	Our main result is that even when there is no information deficiency specific to one group (e.g., both groups have infinite data), adding the same amount of feature noise to all individuals leads to loss discrepancy.
	For linear regression, we thoroughly characterize the effect of feature noise on loss discrepancy in terms of the amount of noise, the difference between moments of the two groups, and whether group information is used or not.
	We then show this loss discrepancy does not vanish immediately if a shift in distribution causes the groups to have similar moments. 
	On three real-world datasets, we show feature noise increases the loss discrepancy if groups have different distributions, while it does not affect the loss discrepancy on datasets where groups have similar distributions.
\end{abstract}

	\section{Introduction}
\label{sec:intro}

Standard learning procedures such as empirical risk minimization have been shown to result in
models that perform well on average but whose performance differ widely across groups such as whites and non-whites \citep{angwin2016machine,barocas2016}.
This \emph{loss discrepancy} across groups
is especially problematic in critical applications that impact people's lives \citep{berk2012criminal,chouldechova2017}.
Despite the vast literature on removing loss discrepancy \citep{hardt2016,khani2019mwld,agarwal2018reductions,zafar2017fairness},
the direct removal of loss discrepancy might introduce other problems such as intra-group loss discrepancy \citep{lipton2018does}
and adverse long-term impacts \citep{liu2018delayed}.
Therefore, it is important to understand the source of loss discrepancy.

Why do such loss discrepancies exist?
The literature generally studies sources of loss discrepancy due to an ``information deficiency'' of one group---that is,
one group has, for example,
more noise \citep{corbett2017algorithmic},
less training data \citep{chouldechova2018frontiers,chen2018my},
biased prediction targets \citep{madras2019fairness},
or less-predictive features \citep{chen2018my}.
Some work also states that groups have different risk distributions,
and thus making hard (binary) decisions on such distributions causes loss discrepancy \citep{corbett2018measure,canetti2019soft}.
In this work, we show that even under very favorable conditions---i.e., no bias in the prediction targets, \emph{infinite} data, perfect predictive features for both groups,
and no hard (binary) decisions (in regression)---adding the \emph{same} amount of feature noise to all individuals still leads to loss discrepancy.

In order to study the effect of feature noise (which includes omitted features as a special case) and use of group information on loss discrepancy, we consider the following regression setup.
We assume each individual belongs to a group $g \in \{0,1\}$ and has latent features $z \in \BR^d$ which cause the target $y = \beta^\top z + \alpha$.
However, we only observe a noisy version of $z$
through one of the following \emph{observation functions}:
\begin{align}
\obnog(z, g, u) &= [z + u],\\
\obg(z, g, u) &= [z + u, g],
\end{align}
where $u \in \BR^d$ is mean-zero noise independent of the rest of the variables,
and the group membership $g$ can be either included or not.
We study the discrepancy of both the residual ($y - \hat y$), which measures the amount of underestimation and the squared error ($(y - \hat y)^2$), which measures the overall performance.
Abusing terminology, we call both \emph{losses}.

We consider two common flavors of loss discrepancies:
(i) \emph{statistical loss discrepancy}, which measures the difference between the expected losses of the two groups
\citep{hardt2016,agarwal2018reductions,woodworth2017,pleiss2017,khani2019mwld};
and (ii) \emph{counterfactual loss discrepancy}, which measures
the difference between the loss of an individual and a ``counterfactual'' individual
with the same characteristics but from another group \citep{kusner2017,chiappa2019path,loftus2018causal,nabi2018fair,kilbertus2017avoiding}.

We have two main results.
First, we show that without using group information, feature noise causes statistical loss discrepancy determined by four factors: the amount of feature noise and the difference between means, variances, and sizes of the groups.
In particular, the loss discrepancy based on residual is proportional to the difference between means,
and the loss discrepancy based on squared error is proportional to the difference between variances (\refprop{obnog}). 
Our second result is that using group information ($\obg$) alleviates the statistical loss discrepancy but causes high counterfactual loss discrepancy (\refprop{obg}).

To better understand the effect of using group information, we further decompose the incurred loss discrepancy into two terms, one related to the moments of training distribution, and one related to the moments of the test distribution.
We show that the statistical loss discrepancy of
$\obnog$ is mainly due to differences in the test distribution, and it vanishes immediately if a shift in distribution causes the groups to have similar distributions.
Meanwhile, the loss discrepancy of $\obg$ is mainly due to differences in the training distribution, and it does not vanish immediately after shifts in the population (\refprop{cov_shift}).

%

We validate our results on three real-world datasets: for predicting the final grade of secondary school students \pl{(dataset name)}, final GPA of law students \pl{(dataset name)}, and crime rates in the US communities \pl{(dataset name)},
where the group $g$ is either race or gender.
We consider two types of feature noise: (i) adding the same amount of noise to every feature and (ii) omitting features.
We show that on the Communities\&Crime and Students datasets where groups have different means, variances, and sizes, noise leads to high loss discrepancy.
On the other hand, on the Law School dataset, where groups have similar means and variances, noise does not affect the loss discrepancy.
Finally, for the datasets with high loss discrepancy,
we consider a distribution shift to a re-weighted dataset where groups have similar means and show that the loss discrepancy of the estimator using $\obnog$ vanishes immediately while the loss discrepancy of the estimator using $\obg$ vanishes more slowly \pl{/(\refprop{cov_shift})} with the rate studied in \refprop{cov_shift}.

\section{Setup}
\label{sec:setup}
We consider the following regression setup, summarized in 
\reffig{causal_graph}.
We assume each individual belongs to a group $g \in \{0,1\}$, e.g., whites and non-whites; and has latent (unobservable) features, $z \in \BR ^d$ which cause the prediction target $y \in \BR$.
For each individual, we observe $x = \ob(z, g, u)$ through an observation function $\ob$,
where $u \in \BR^d$ is a random vector representing the source of (feature) noise in observation. 
\pl{maybe should say parenthetically or in a footnote something about target noise?}

As an example, the latent feature ($z$) is the knowledge of a student in $d$ subjects, and $y$ is her score in an entrance exam, which is a combination of the different subjects.
However, we only observe a noisy version of $z$ via exam scores, the school's name, or letter of recommendation, where the latter two might reveal information about group membership ($g$). 

Let $h : \BR^d \to \BR$ be a predictor, and $\hat y = h(o(z,g,u))$ be the predicted value for individual $z$.
We measure the impact of the predictor for an individual through a loss function $\ell(\hat y, y)$ (e.g., for squared error, $\ell(\hat y, y) = (\hat y - y)^2$),
abbreviated as $\ell$ when clear from the context.
In the entire paper, we analyze the population setting (infinite data) since we show that the effect of feature noise does not even vanish in this favorable setting.
\begin{figure}
	\tikzset{
		-Latex,auto,node distance =1 cm and 1 cm,semithick,
		state/.style ={circle, draw, minimum width = 0.5 cm, inner sep=0pt},
		fstate/.style ={circle, draw, minimum width = 0.5 cm, inner sep=0pt, fill=gray!50},
		point/.style = {circle, draw, inner sep=0.04cm,fill,node contents={}},
		bidirected/.style={Latex-Latex,dashed},
		el/.style = {inner sep=2pt, align=left, sloped}
	}
	\begin{minipage}[t]{0.14\textwidth}
		\vspace{0pt}
		\begin{tikzpicture}
		\node[fstate] (x) at  (0,    0) {$x$};
		\node[fstate] (hy) at (0,    -0.9) {$\widehat y$};
		\node[fstate] (y) at  (-1.4, -0.9) {$y$};
		\node[fstate] (g) at  (-0.7,    0.8) {$g$};
		\node[state] (z) at   (-1.4, 0) {$z$};

		\path (z) edge (y);
		\path (g) edge (x);
		\path (z) edge (x);
		\path (g) edge (z);
		\path (x) edge (hy);
		\end{tikzpicture}
\end{minipage}%
	\begin{minipage}[t]{0.34\textwidth}
	\caption{\label{fig:causal_graph} In this work, we consider a prediction problem from $x$ to $\hat y$ where the output $y$ is a deterministic function of unobserved random vector $z$.}
	\end{minipage}
\end{figure}
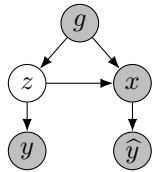

%

\subsection{Loss Discrepancy Notions}
There are two flavors of loss discrepancies:
statistical and counterfactual loss discrepancy.
Statistical loss discrepancy measures how much two groups are impacted differently. This notion has been studied in economics and machine learning under the names of disparate impact, equal opportunity, classification parity, etc. \citep{arrow1973theory,phelps1972statistical,hardt2016,corbett2018measure}.
Here, we define statistical loss discrepancy similar to mentioned work, but looking at a general loss function.
\begin{definition}[Statistical Loss Discrepancy (\sld{})]    
	\label{def:sld}
	For a predictor $h$, observation function $o$, and loss function $\ell$,
	statistical loss discrepancy is the difference between the expected loss between two groups:
	\begin{align}
	\sld(h,\ob, \ell) = \pab{\E [\ell \mid g=1] - \E [\ell \mid g=0]}
	\end{align}
\end{definition}

\sld{} operates at the group level and indicates how much a predictor yields higher loss for one group.
However, it does not provide any guarantees at the individual level.

For individuals,
counterfactual loss discrepancy measures how much two similar individuals (in our setup, two individuals that have the same $z$,  not necessarily the same $x$) are treated differently because of their group membership.
Here, we define counterfactual loss discrepancy similar to prior work \cite{kusner2017,nabi2018fair}, but looking at a general loss function.
\begin{definition}(Counterfactual Loss Discrepancy (\cld{}))
	\label{def:cld} 
	For a predictor $h$, observation function $o$, and loss function $\ell$,
	counterfactual loss discrepancy is the expected difference between the loss of an individual and its counterfactual counterpart:
	\begin{align}
	\cld (h, \ob, \ell) = \E \pb{\pab{L_0 - L_1}},
	\end{align}
	where $L_{g’} = \E[\ell(h(o(z, g’,u)), y) | z]$.
\end{definition}

There are many concerns regarding the meaningfulness of CLD when group identity is an immutable characteristic (e.g., race and sex) \cite{holland1986statistics, freedman2004graphical,holland2003causation}, which we discuss further in \refsec{related_work}.

Note that
\cld{} and \sld{} are not comparable.
A model can have the same loss for similar individuals ($\cld=0$), but since groups have different distributions over individuals, it can have higher loss for one group ($\sld\neq0$).
Conversely, a model can induce different losses for similar individuals due to their group membership ($\cld\neq0$), but when averaged over the groups, it can result in similar expected losses for both groups ($\sld=0$). See \refprop{not_comparable} for the construction.

It is clear that the observation function can asymmetrically affect groups and cause high loss discrepancy.
For example, the observation function ($\ob$) can add noise to the features \emph {only} when $g=0$ or systematically report a lower value of $z$ for one group. 
However, in this work, we are interested to see if it is possible that adding the same amount of noise (in a symmetric way) to all individuals affects groups differently (i.e., causes high loss discrepancy). 
We answer this question in the affirmative and exactly characterize the group distributions that are more susceptible to have high loss discrepancy under feature noise.

	\begin{figure}[t]
	\centering
	\includegraphics[width=0.5\textwidth]{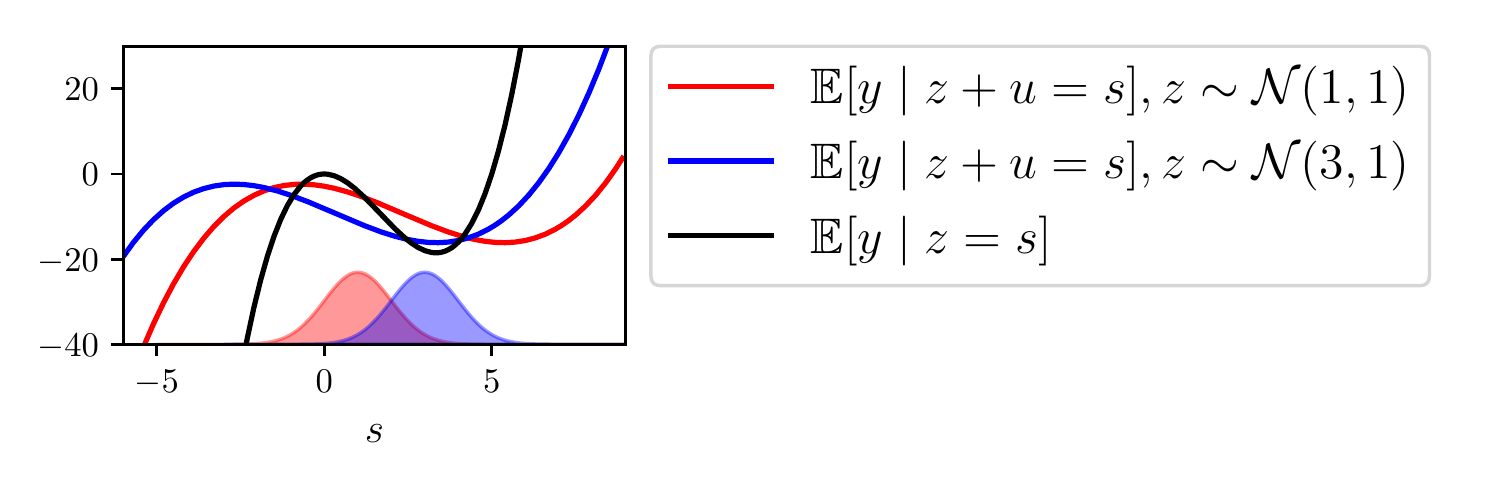}
	\caption{\label{fig:mes_error_non_linear} 	
	In this example $y = f(z) = z^3+5z^2$, 
	and we observe $z+u$ instead of $z$ where $u \sim \sN(0,1)$. 
  As shown in \refeqn{dependence}, the best estimate of $y$ (i.e., $\E [y \mid z+u]$) depends on the distribution of $z$.
	The blue line is the best estimate  when $z \sim \sN (1,1)$ and the red line is the best estimate when $z \sim \sN(3,1)$.}
\end{figure}

\section{Feature Noise}
\label{sec:feature_noise}
We are interested in additive feature noise with mean zero that is independent of other variables ($z$ and $y$). Feature noise is also known as classical measurement error \cite{carroll2006measurement}.
We allow for any noise distribution---e.g., Laplace, Gaussian, or any discrete distribution---as long as it has mean zero. 
The independence assumption means $u$ is independent of the value of $z_i$, but the feature noise can have different distributions for different features.
Omitted features can be simulated by having noise with infinite variance.
Feature noise and omitted features are pervasive in real-world applications; examples include test scores for college admissions or interview scores for hiring.

Note that without feature noise, i.e., $x=z$, the Bayes optimal predictor, $\E[y \mid x]$, does not depend on the distribution of $z$, but this no longer holds with feature noise.
Formally, let $y=f(z)$ and $u$ denote the additive noise on each feature, i.e., we observe $x=z+u$ instead of $z$.
In this case, the Bayes-optimal predictor $\E [y \mid x]$, depends on the distribution of inputs ($\pr_z$):
\begin{align}
\label{eqn:dependence}
\E [y \mid x] = \frac{\int \pr_u (u) \pr_z (x-u) f (x-u)du}{\int \pr_u (u) \pr_z (x-u)du}.
\end{align}
\reffig{mes_error_non_linear} shows an example of this dependence.

Feature noise has been extensively studied in linear regression (e.g., \citet{fuller2009measurement}).
\pl{say something more about relationship to our paper}
In the rest of the paper, we focus on linear regression and show feature noise can cause loss discrepancy across groups.

\subsection{Feature Noise in Linear Regression Background}
\label{sec:measurement_error}
\pl{In this section, we specialize to linear regression...}
In this section, we give a brief background on how feature noise makes parameter estimates inconsistent, in the simplified setting. 
We study the effect of feature noise on groups in \refsec{ld_linear_regression}.
Let $\beta, \alpha$ denote the true parameters such that for each individual, $y=\beta^\top z + \alpha$, 
\footnote{Observing a noisy version of $y$ does not change the estimate of parameters in presence of infinite data. For simplicity, we consider noiseless $y$.}
and assume we observe $x = z+u$. 
When $u$ is feature noise (i.e., mean-zero and independent of other variables), we can analyze the estimated parameters via least squares \citep{frisch1934statistical}.
\begin{align}
\label{eqn:measurement_error_beta}
\hat \beta =& \Sigma_x^{-1}\Sigma_{xy} =(\Sigma_z + \Sigma_u)^{-1}\Sigma_z \beta\\
\label{eqn:measurement_error_alpha}
\hat{\alpha} =& (\beta - \hat \beta)^\top \E [z] + \alpha,
\end{align}
where for any two random vectors $v$ and $w$, $\Sigma_{vw} = \E [ (v - \mu_v)(w - \mu_w)^\top]$ 
denotes the covariance matrix between $v$ and $w$, and $\mu_v = \E [v]$ denote the expected value of $v$. We write $\Sigma_v$ for $\Sigma_{vv}$. To simplify notation,
let $\Lambda \eqdef (\Sigma_z + \Sigma_u)^{-1}\Sigma_u$ denote the noise to signal ratio, then $\hat \beta = (I-\Lambda)\beta$ and $\hat \alpha =  (\Lambda \beta)^\top \E [z]+\alpha$. 
Finally, for these estimated parameters the squared error is:
\begin{align}
\label{eqn:mes_error_variance}
(\Lambda\beta)^\top \Sigma_z\Lambda \beta + ((I-\Lambda)\beta)^\top\Sigma_u((I-\Lambda)\beta).
\end{align}
Note that the actual estimator only has access to $x$, but our analysis is in terms of $z$ and $u$.
If all variables are one-dimensional, then $\hat \beta = \frac{\Sigma_z}{\Sigma_z + \Sigma_u} \beta < \beta$, where $\frac{\Sigma_z}{\Sigma_z + \Sigma_u}$ is the relative size of the true signal and is known as attenuation bias (See \citet{wager2013dropout} for a connection between regularization and feature noise)
\reffig{measurement_error} shows the estimated line which predicts $y$ from $x$ in comparison to the true line which predicts $y$ from $z$.
\begin{figure}[t]
	\centering
		\includegraphics[width=0.38\textwidth]{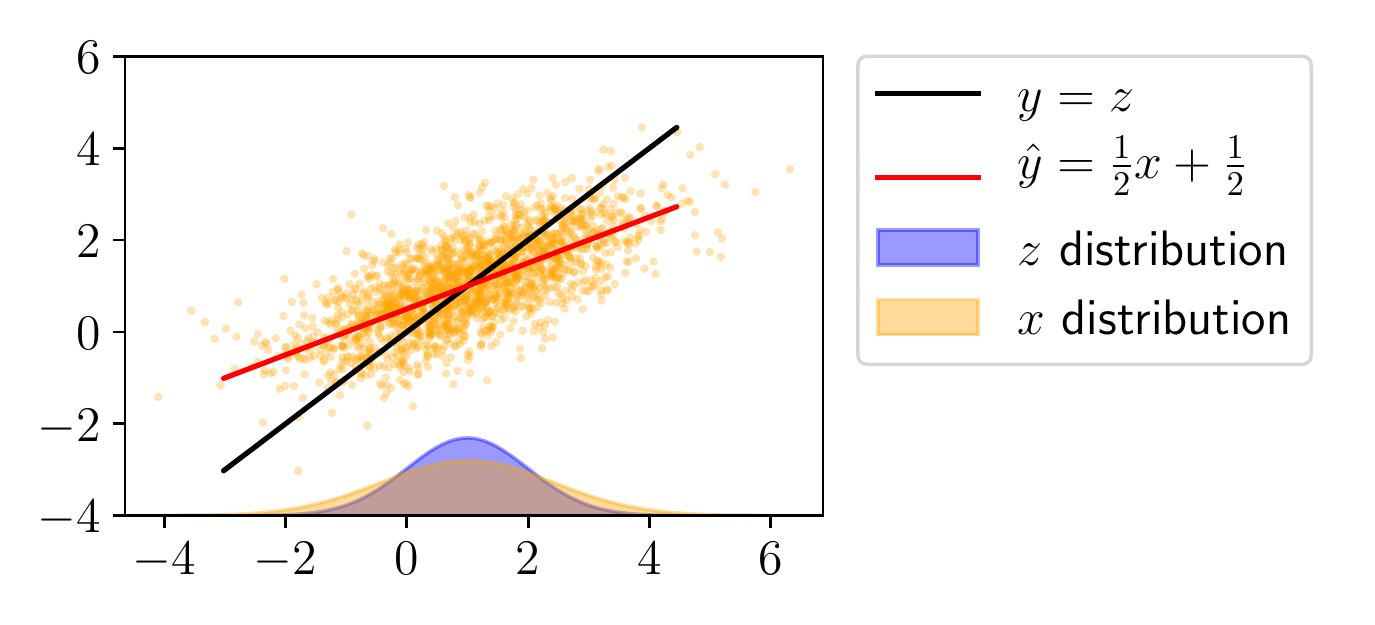}
		\caption{\label{fig:measurement_error}In the presence of feature noise, least squares estimator is not consistent; and the estimated slope (red line) is smaller than the true slope (black line).
		Here the true feature is $z \sim \sN (1,1)$, the observed features is $x \sim \sN(z,1)$, and the prediction target $y=z$. }
\end{figure}

	\section{\cld{} and \sld{} for Linear Regression}
\label{sec:ld_linear_regression} 
We now show how feature noise affects \sld{} (\refdef{sld}) and \cld{} (\refdef{cld}) for linear regression.
We focus on two loss functions when computing \cld{} and \sld{}:
\begin{itemize}[nosep]
	\item Residual: measures the amount of underestimation.
	\begin{align}
	\lr(\hat y,y) = y - \hat y.
	\end{align}
	\item Squared error: measures overall performance.
	\begin{align}
	\ls(\hat y, y) = (y - \hat y)^2.
	\end{align}
\end{itemize}
In this section, we calculate eight metrics according to different notions of loss discrepancy (CLD and SLD), different losses ($\lr$ and $\ls$), and whether group information is used or not ($\obg$ and $\obnog$).
\reftab{summary} demonstrates three points:
1) In the presence of feature noise, SLD is not zero.
2) Using group membership reduces SLD, but increases CLD.
3) Groups are more susceptible to loss discrepancy based on residual when they have different means; they are more susceptible to loss discrepancy based on squared error when they have different variances.

\begin{figure}
	\centering
	\includegraphics[width=0.4\textwidth]{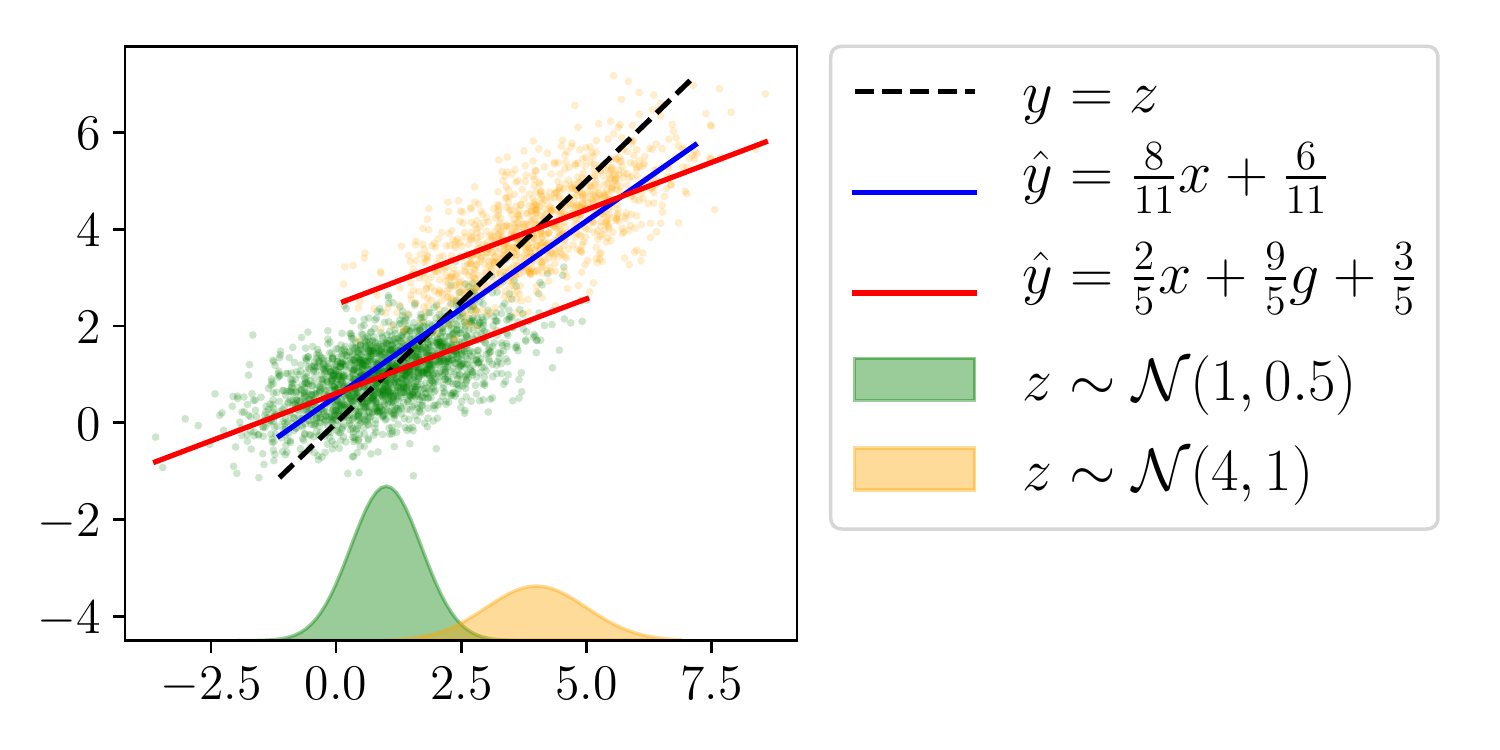}
	\caption{\label{fig:measurement_error_two_groups}
		An illustration of feature noise and its effect on \cld{} and \sld{}.
		There are two groups: green and orange.
		The true function (dashed black line) is $y=z$. The feature noise $u \sim \sN(0,1)$.
    Predicting $y$ from $x=\obnog(z,g,u) = z + u$ is the blue line which underestimates the target values for the orange group and thus has high $\sld(\obnog,\lr)$; however, since the prediction is independent of the group membership, it has $\cld=0$.
      Predicting $y$ from $\obg(z,g,u) = [z+u, g]$ is the red lines, which has high CLD since groups are treated differently but low SLD.}
\end{figure}


\subsection{Effect of Noise Without Group Information}
\label{sec:obnog}
In the entrance exam example in \refsec{setup}, suppose we observe each students' exam performance, which is a noisy version of their true knowledge of a subject.
How does this noise affect the prediction? Is it possible that this symmetric independent noise over all features and all individuals affect groups differently?
\pl{turn second question into summary of the result to not keep readers in suspense}

In this section, we show observing a noisy version of $z$ without any information about group membership ($g$) leads to high SLD.
Formally, let $u$ denote the feature noise, we define the following observation function, 
\begin{align}
\label{eqn:obnog}
\obnog(z,g,u) \eqdef z + u.
\end{align}
In this case, group information is not encoded in the observation function ($\obnog(z,0,u) = \obnog(z,1,u)$); therefore, $\cld=0$.
But, as we show, SLD depends on the distribution of $z$.

Let's first consider a simple one-dimensional case.
\reffig{measurement_error_two_groups} shows two groups, where $z \sim \sN(1,0.5)$ for the green group ($g=0$), and $z \sim \sN(4,1)$ for the orange group ($g=1$), also the $g=1$ group is twice as likely as the $g=0$ group ($\pr [g=1] = 2 \pr [g=0]$).
The prediction target here is $y=z$.
Let the  noise be Gaussian $u \sim \sN(0,1)$, and we observe $x = \obnog (z,g,u) = z + u$.
Concretely, we are interested in the statistical loss discrepancy between groups for the least squares estimator, which predicts $y$ from $x$.

As shown in \refsec{measurement_error}, having feature noise results in attenuation bias.
In this example, we have $\var [z] = \E [\var [z \mid g]] + \var [\E [z \mid g]] = \frac{8}{3}$.
Therefore, $\hat \beta = \frac{\Sigma_z}{\Sigma_z + \Sigma_u} = \frac{8}{11}$, and $\hat \alpha = \frac{6}{11}$ (see blue line in \reffig{measurement_error_two_groups}).

Let's see how this attenuation bias affects different groups. 
As shown in the \reffig{measurement_error_two_groups}, the prediction target for the orange group is underestimated.
Intuitively, if the mean of a group deviates from the mean of the population, then the expected residual for that group is large. 
Therefore, the difference between means of the groups is a factor in loss discrepancy based on residual (\lr).
\begin{align}
\label{eqn:dif_mean}
\Delta\mu_z \eqdef \E [z \mid g=1] - \E [z \mid g=0].
\end{align}
Secondly, since the green group is in the majority ($\pr [g=1] > \pr [g=0]$), the line has less bias for the green group. 
Hence, the difference between size of the groups also plays an important role in loss discrepancy.
\begin{align}
\label{eqn:dif_size}
\pr [g=1] - \pr [g=0] 
\end{align}

Thirdly, as shown in \ref{eqn:mes_error_variance}, the squared error is related to variance of data points; so intuitively, the difference in variance should also be a main factor in loss discrepancy.
\begin{align}
\label{eqn:dif_var}
\Delta \Sigma_z \eqdef \var [z \mid g=1] - \var [z \mid g=0].
\end{align}
Finally, as noise increases, the attenuation bias increases, thus the estimated line deviates more from the true line, leading to a higher loss discrepancy.
The following proposition formalizes how \sld{} depends on the four factors above; see Appendix~\ref{sec:obnog_proof} for the proof.
\begin{restatable}{proposition}{propObnog}
	\label{prop:obnog}
	Consider the observation function $\obnog$ \refeqn{obnog}.
	Let $\Lambda \eqdef (\Sigma_z + \Sigma_u)^{-1}\Sigma_u$.
	The loss discrepancies for least squares estimator are as follows:
	{\small\begin{align*}
		\cld (\obnog, \lr ) = & \, \cld (\obnog, \ls) = 0 \nonumber\\
		\sld (\obnog, \lr ) = &\pab {(\Lambda\beta)^\top \Delta\mu_z} \nonumber\\
		\sld (\obnog, \ls) =  &\Big| (\Lambda\beta)^\top \Delta\Sigma_z(\Lambda\beta) \\ &-(\pr [g=1] - \pr [g=0])((\Lambda\beta)^\top\Delta\mu_z)^2\Big |.
		\end{align*}}
	where $\Delta\mu_z$ and $\Delta\Sigma_z$ are as defined in \refeqn{dif_mean} and \refeqn{dif_var}.
\end{restatable}
\pl{If you swap $\BP[g = 1]$ and $\BP[g = 0]$, it changes the SLD..., I guess it's not symmetric?}

\refprop{obnog} states that SLD is not zero in the presence of feature noise.
Furthermore, it determines the  characteristics of the group distributions which are more prone to incur high SLD.
In particular, given fixed variance in $z$ (therefore, fixed $\Lambda$),
 groups with higher difference in means are more susceptible to incur high SLD based on residuals.
SLD based on squared error has two terms: the first term is related to the difference between variance of the groups, and the second term is non-zero if groups have different sizes.
We observe the effect of the second term in a real-world dataset in \refsec{experiments}.

\begin{table*}
	\centering
	\scalebox{0.9}{
		\begin{tabular}{lll|l}
			\toprule
			&& \multicolumn{1}{c|}{Counterfactual Loss Discrepancy (\cld{})}& \multicolumn{1}{c}{Statistical Loss Discrepancy (\sld{})}\\ \midrule
			\multirow{2}{*}{$\bf \boldsymbol{\lr}$} &$\obnog$ & $0$ & $|(\Lambda\beta)^\top \Delta\mu_z|$\\
			&$\obg$ & $|(\Lambda'\beta)^\top \Delta\mu_z|$& $0$\\ \midrule 
			\multirow{2}{*}{$\bf \boldsymbol{\ls}$} & $\obnog$&  $0$ & $|(\Lambda\beta)^\top \Delta\Sigma_z (\Lambda\beta) - (\pr [g=1] - \pr [g=0])((\Lambda\beta)^\top\Delta\mu_z)^2|$\\
			&$\obg$&$|(\Lambda'\beta)^\top\Delta\mu_z|\E [|(\Lambda'\beta)^\top (2z-\mu_1 - \mu_2)|]$& $|(\Lambda'\beta)^\top \Delta\Sigma_z (\Lambda'\beta)|$\\\bottomrule
		\end{tabular}}
		\caption{\label{tab:summary} Loss discrepancies between groups,
			as proved in \refprop{obnog} and \ref{prop:obg}.
			In summary:
			1. Feature noise without group information ($\obnog$) causes high SLD (first and third row),
			2. Using group information reduces SLD but increases CLD (second and forth row), and
			3. In loss discrepancies based on residuals the difference between mean is important while for squared error the difference between variances is important. 
		}
	\end{table*}

\subsection{Effect of Noise with Group Information}
\label{sec:obg}

Now let us consider the setting where the predictor is allowed to use group information.
Does the predictor put weight on group membership information (e.g., assigns non-zero weight on the group membership feature),
or does the predictor ignore the group membership since all the necessary information ($z$) is available and we are in the infinite data limit?

In this section, we show that if the observation function reveals the group information, as feature noise increases, the estimator relies more on group information (thus resulting in high CLD). 
On the other hand, the reliance on group information alleviates SLD.

Formally, we define a new observation function that adds group membership $g$ as a separate feature:
\begin{align}
\label{eqn:obg}
\obg(z,g,u) \eqdef [z + u, g]
\end{align}
Let's first revisit the example in \reffig{measurement_error_two_groups}.
The goal is to predict $y$ (where in this example, we simply have $y=z$).
The red lines indicate the estimated line that the least squares estimator predicts for $y$
given a noisy version of $z$ and the group membership ($x = \obg (z,g,u) = [z+u, g]$).
In this case, having $g$ as an additional feature enables the model to have different intercepts for each group.
As a result, the average residual for each group is zero; therefore $\sld(\obg, \lr )=0$.
However, this benefit comes at the expense of treating individuals with the same $z$ differently.

For the squared error ($\ls$), since each group has its own intercept,
the squared error is no longer related to difference in sizes or means.
The following proposition characterize \cld{} and \sld{} for the least squares estimator using $\obg$.
See Appendix~\ref{sec:obg_proof} for the proof. 
\begin{restatable}{proposition}{propObg}
	\label{prop:obg}
	Consider the observation function $\obg$ \refeqn{obg}.
	Let $\Sigma_{z\mid g} = \E [\var [z \mid g]]$, and $\Lambda' =(\Sigma_{z\mid g} + \Sigma_u)^{-1}\Sigma_u$.
	The estimated parameters using least squares estimator are:
	\begin{align*}
	&\hat{\beta} = \begin{bmatrix}
	(I - \Lambda')\beta\\
	(\Lambda'\beta)^\top \Delta\mu_z
	\end{bmatrix},
	&\hat\alpha = (\Lambda'\beta)^\top\E[z\mid g=0] + \alpha.
	\end{align*}
	The loss discrepancies are as follows:
	{\small\begin{align*}
		&\cld (\obg, \lr ) = \pab {(\Lambda'\beta)^\top\Delta\mu_z}\\
		&\cld (\obg, \ls ) = \pab{(\Lambda'\beta)^\top\Delta\mu_z}\E \pb{\pab{(\Lambda'\beta)^\top\p{2z -\mu_1-\mu_0}}}\\
		&\sld (\obg, \lr ) = 0\\
		&\sld (\obg, \ls ) = \pab {(\Lambda'\beta)^\top\Delta\Sigma_z (\Lambda'\beta)},
		\end{align*}}
	where $\Delta\mu_z$ and $\Delta\Sigma_z$ are as defined in \refeqn{dif_mean} and \refeqn{dif_var}, and $\mu_1 \eqdef \E [z \mid g=1]$ and $\mu_0 \eqdef \E [z \mid g=0]$.
\end{restatable}
\refprop{obg} states that the coefficient for group membership is $\betg = (\Lambda'\beta)^\top\Delta\mu_z$ (note that this is similar to $\sld (\obnog, \lr)$).
By having this coefficient for $g$, the estimator has $\sld(\obg, \lr )=0$, but it results in $\cld (\obg, \lr) = |\betg|$.
%
Returning to the example in Figure~\ref{fig:measurement_error_two_groups}, 
note that the red lines have better performance for each group, both in terms of residuals and squared error.

\reftab{summary} presents a summary of the 8 different types of loss discrepancies. 
We also study non independent noise in Appendix~\ref{sec:general_noise}, and infinite noise in Appendix~\ref{sec:infinite_noise}.

%

\section{Persistence of Loss Discrepancy}
\label{sec:cov_shift}
So far, we assumed the training distribution used to estimate parameters is the same as the test distribution that we are interested in measuring loss discrepancy with respect to.
But what if the train and test distributions are different?
Our formulation presented in \refprop{obnog} and~\ref{prop:obg} can be refined to be in terms of train and test distributions as follows (for the sake of space, we only focus on residual loss  $\lr$),

\vspace{-0.4cm}
{\small	\begin{align}
		\cld (\obnog, \lr) &= 0   \label{eqn:decompose_cldnog}\\
	\sld (\obnog, \lr) &= \pab {(\Lambda_{\text{(train)}}\beta)^\top \Delta\mu_{z(\text{test})}} \label{eqn:decompose_sldnog}\\
	\cld(\obg, \lr) &= \pab{ (\Lambda'_{\text{(train)}}\beta)^\top\Delta\mu_{z(\text{train})}} \label{eqn:decompose_cldg}  \\
	\sld (\obg, \lr) &= \pab {(\Lambda'_{\text{(train)}}\beta)^\top (\Delta\mu_{z(\text{train})} - \Delta\mu_{z(\text{test})})}, \label{eqn:decompose_sldg}
	\end{align}}%
where the subscript denotes whether the statistics are computed on the training or test distribution.
%

When group information is not used then $\cld=0$ for any test distribution \refeqn{decompose_cldnog}, and the incurred \sld{} of $\obnog$ is due to the difference in means of the groups in the test distribution and will vanish if groups starts to have same means due to covariate shift \refeqn{decompose_sldnog}. 
On the other hand, when group information is used then $\cld{}\neq 0$ and it is proportional to the difference in the means of the groups in the \emph {training} distribution \refeqn{decompose_cldg}. 
Furthermore, in \refprop{obg} we showed when group membership is used then $\sld (\obg, \lr) = 0$; however, this will no longer holds when $\Delta\mu_{z(\text{train})} \neq \Delta\mu_{z(\text{test})}$ due to the covariate shift \refeqn{decompose_sldg}. 
In summary, using $\obg$ leads the loss discrepancies of the linear predictor to be more dependent on the training data, thus more persistent even when groups start to have the same means due to a covariate shift.


To study the persistence of loss discrepancy, we instantiate the above expressions in following simple setting.
We consider two distributions:
\begin{itemize}[nosep]
  \item {Initial distribution:} The mean of $z$ for group $g=1$ is $-\mu$, and for group $g=0$ is $\mu$, the covariance of $z$ for both groups is $\Sigma$.
	\item {Shifted distribution:} The mean of $z$ for both groups is $\mu$ and its covariance is $\Sigma$.
\end{itemize}
Here we assume groups have the same covariances/sizes, but the same analysis applies more generally.
\pl{need a better transition to this proposition, especially since we've been assuming infinite data everywhere, so what does 'more data' even mean? say more precisely, relative fraction/weighting}
The following proposition studies the persistence of loss discrepancies as we see data from the shifted distribution with higher probability.

\begin{restatable}{proposition}{covShift} 
	\label{prop:cov_shift}
  For each $0\le t \le 1$, let the training distribution be a mixture of the initial distribution with probability $t$ and the shifted distribution with probability $1-t$. 
  Let \pl{any English interpretation for these two quantities? how are they related} $c_1 = \p {({\Sigma+\Sigma_u})^{-1}\Sigma_u\beta}^\top(2\mu)$, $c_2 = \p {(\Sigma+\Sigma_u)^{-1}\mu\mu^\top (\Sigma+\Sigma_u)^{-1}\Sigma_u\beta}^\top(2\mu)$.
  For a linear predictor which is trained on the above distribution and tested on the shifted distribution, we have:
	\begin{align}
	t \p {c_1 - |c_2| }\le &\sld(\obg, \lr) =  \nonumber \\
	&\cld (\obg, \lr) \le t \p {c_1 + |c_2|} \\
	\sld (\obnog, \lr) &= \cld (\obnog, \lr)=0 .
	\end{align}
\end{restatable}
One way to interpret this proposition is as follows. 
We start with a batch from initial distribution, at each time step we predict the targets for the shifted distribution and then concatenate the new batch with the correct targets to the training data. 
At time $K$, the  training data consists of $K+1$ batches, where one batch is from the initial distribution with $\sld_\text{initial} (\obg , \lr)$ and $K$ batches from the shifted distribution.
In \refprop{cov_shift} terms,  we can assume the training data is a mixture of initial distribution with probability $t=\frac{1}{K}$ and the shifted distribution with probability of $1-t = \frac{K}{K+1}$.
The loss discrepancy on the shifted distribution is $\sld_\text{new} (\obg, \lr) \approx \frac{1}{K+1}\sld_\text{initial} (\obg, \lr)$, which converges to zero with rate $O(\frac1K)$.
For $\obnog$, we have $\cld (\obnog, \lr) = \sld (\obnog, \lr)=0$ for all $K$. See Appendix~\ref{sec:cov_shift_proof} for the proof.


	\begin{table*}
	\centering
	\scalebox{0.82}{
	\begin{tabular}{lrrlllrrrrr} \toprule 
		\bf name&\bf \#records &\bf \#features &\bf target &\bf features example & \bf group  &\bf  $\boldsymbol{\pr [g=1]}$ & $\boldsymbol{\Delta\mu_y}$ &  $\boldsymbol{\Delta\sigma_y^2}$ & $\|\boldsymbol{\Delta\mu_x}\|_2$ &  $\|\boldsymbol{\Delta\Sigma_x}\|_F$\\ \midrule
	\bf C\&C &1994&91&crime rate&\#homeless, average income, \dots& race & 0.50 & 1.10 &0.96 & 5.62 & 12.75\\ \midrule
	\multirow{2}{*}{\bf law}&	\multirow{2}{*}{20798}&	\multirow{2}{*}{25}& \multirow{2}{*}{final GPA}&\multirow{2}{*}{undergraduate GPA, LSAT, \dots} &race & 0.86 & 0.87 & 0.01 & 2.24 & 2.79\\ 
 &&&&& sex & 0.56 & 0.005 &0.04 &  0.42 & 0.51\\\midrule
\bf	students &649&33&final grade& study time, \#absences, \dots& sex & 0.59 & 0.26 & 0.12 & 1.40 & 2.26\\ \bottomrule
	\end{tabular}}
	\caption{\label{tab:datasets_description} Statistics of the used datasets. Size of the first group is denoted by $\pr[g=1]$ and  $\Delta\mu_y$ and $\Delta\sigma_y^2$ denote the difference of mean and variance of the prediction target between groups, respectively.}
\end{table*}
{\begin{figure*}[t]
		\begin{subfigure}[b]{\textwidth}
	\centering
	\includegraphics[width=\textwidth]{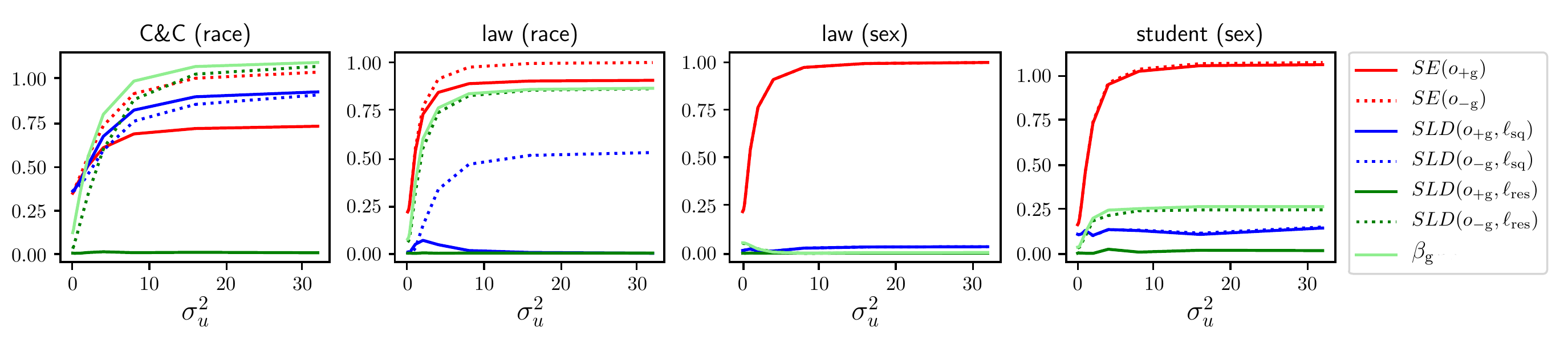}
		\caption{\label{fig:real_norm}
		Adding noise increases squared error (SE) in all datasets; however, noise induces different loss discrepancy across the datasets.
}
\end{subfigure}
\begin{subfigure}[b]{\textwidth}
	\centering
	\includegraphics[width=\textwidth]{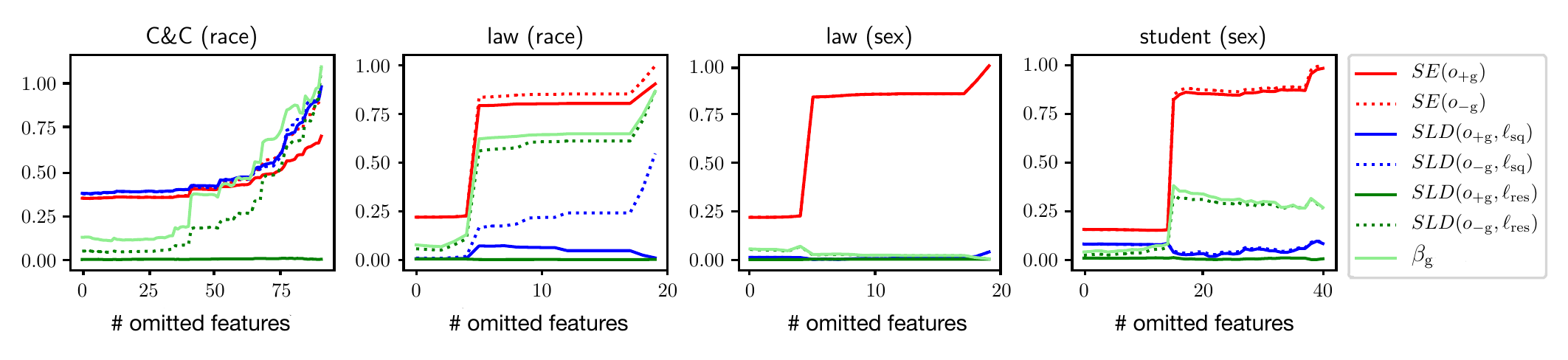}
	\caption{\label{fig:real_drop}
	In all datasets except law(sex), omitting features affects groups differently and causes high SLD.}
\end{subfigure}
\caption{\label{fig:real_both} Statistical loss discrepancy (SLD) and squared error (SE) when (a) independent normal noise ($u \in \sN(0, \sigma_u^2)$), is added to each feature (except for the group membership) (b) normal noise with high variance is added to the features sequentially (except for the group membership). We report $\betg$ as a proxy for $\cld(\obg, \lr)$.}
\end{figure*}}


\section{Experiments}
\label{sec:experiments}
{\bf Datasets.}
We consider three real-world datasets from the fairness literature. 
See \reftab{datasets_description} for a summary and Appendix~\ref{sec:datasets} for more details.

%
\pl{Start by saying what we're doing: taking the features as is and pretending these are $z$ and then adding noise on top}
Our assumptions do not hold in these datasets: 
the features are not ideal (they might have information deficiency specific to one group), the model is misspecified (it is not linear), groups might have different true functions.
However, we are still interested to see if adding noise on top of these (not ideal) features impacts groups differently, or whether the loss discrepancy remains the same as its initial value \pl{doesn't add much, cut}.
We observe that the difference between moments of the groups is still a relevant factor governing loss discrepancy; and in the presence of feature noise, datasets where groups have different means, variances, and sizes are more susceptible to loss discrepancy.



{\bf Setup.}
We standardize all features and the target in all datasets (except the group membership feature) to have mean $0$ and variance $1$.
We run each experiment $100$ times, each time randomly performing a 80--20 train-test split of the data, and reporting the average on the test set.
We compute the least squares estimator for each of the two observation functions: $\obnog$ which only have access to non-group features, and $\obg$ which have access to all features.
We consider two types of noise:
\begin{enumerate}[nosep]
  \item {Equal noise}: for different values of $\sigma_u^2$ we add independent normal noise ($u \sim \sN(0, \sigma_u^2)$) to each feature except the group membership.
\item Omitting features:
  We start with a random order of the non-group features and omit features, which is nearly equivalent to adding normal noise with a very high variance ($u \sim \sN(0,10000)$) to them sequentially.
\end{enumerate}


\paragraph{Loss discrepancy based on squared error.}
As expected, increasing the noise results in larger squared errors (SE).
We see a smoother increase in \reffig{real_norm}, as opposed to the large jumps in \reffig{real_drop}, related to the ``importance'' of a feature.

We are now interested to see whether the observed increase differs across groups, thus inducing high loss discrepancy.
In C\&C, as we increase the amount of feature noise ($\sigma_u^2$), SLD increases (blue lines in \reffig{real_norm}), meaning that one group incurs higher loss compared to the other group. 
In law (race) dataset, the sizes of the groups are very different (as shown in \reftab{datasets_description}, whites represent $86\%$ of the population).
Recalling from \refprop{obnog}, when group membership is not used ($\obnog$), the minority group incurs higher loss; this is reflected in the observation that $\sld (\obnog, \ls)$ (dotted blue line) increases as we add more noise.
On the other hand, once group membership is used, the group size does not influence the loss discrepancy; therefore, we do not observe an increase in $\sld (\obg, \ls)$ (see the solid blue line).
In law (sex) and students dataset, since groups have similar variance and sizes, we do not observe increase in loss discrepancy.
\pl{in general, can we standardize on the names of the datasets and use them throughout - "law (race)" isn't obviously a dataset name...you should put them all behind macros,
and do something like \textsc{Law-race} or \textsc{Law(race)}?
}


\paragraph{Loss discrepancy based on residual.}
When we estimate the parameters of linear regression, the bias (average residual) is always zero.
Therefore, as we increase the noise, the average residual remains zero. 

Is the average residual also zero for both groups or does adding noise cause a systematical over/underestimate for some groups (i.e., inducing loss discrepancy based on residual)?
As discussed in \refsec{obg}, when group membership is used, then the average residual for each group is always zero ($\sld(\obg,\lr)=0$)---see the solid green line in \reffig{real_both}. 
 However, if group membership is not used ($\obnog$), as discussed in \refsec{obnog}, feature noise affects groups with different means differently and causes high loss discrepancy based on residuals (see dotted green line in \reffig{real_both}).
The only dataset in which the loss discrepancy for residuals does not increase is law (sex), in which the considered groups have similar means.


Finally, as shown in \reffig{real_both}, weight of the group membership feature ($\betg$) increases as we increase the feature noise in datasets where groups have different means.
Under the strong assumption that the observed features are the same for individuals from different groups (e.g., $x=z+u$), then $\cld (\obg, \lr) = |\betg|$.
As shown in \reftab{summary}, there is a close relationship between $\cld(\obg, \lr )$ and $\sld (\obnog, \lr)$.
Although this assumption does not hold in practice, we still observe $\betg$ is close to $\sld (\obnog, \lr)$ (see the dotted green line and the light solid green line in  \reffig{real_both}).


\paragraph{Persistence of loss discrepancy.}
\label{sec:cov_shift_experiment}
\begin{figure}[t]
	\includegraphics[width=0.5\textwidth]{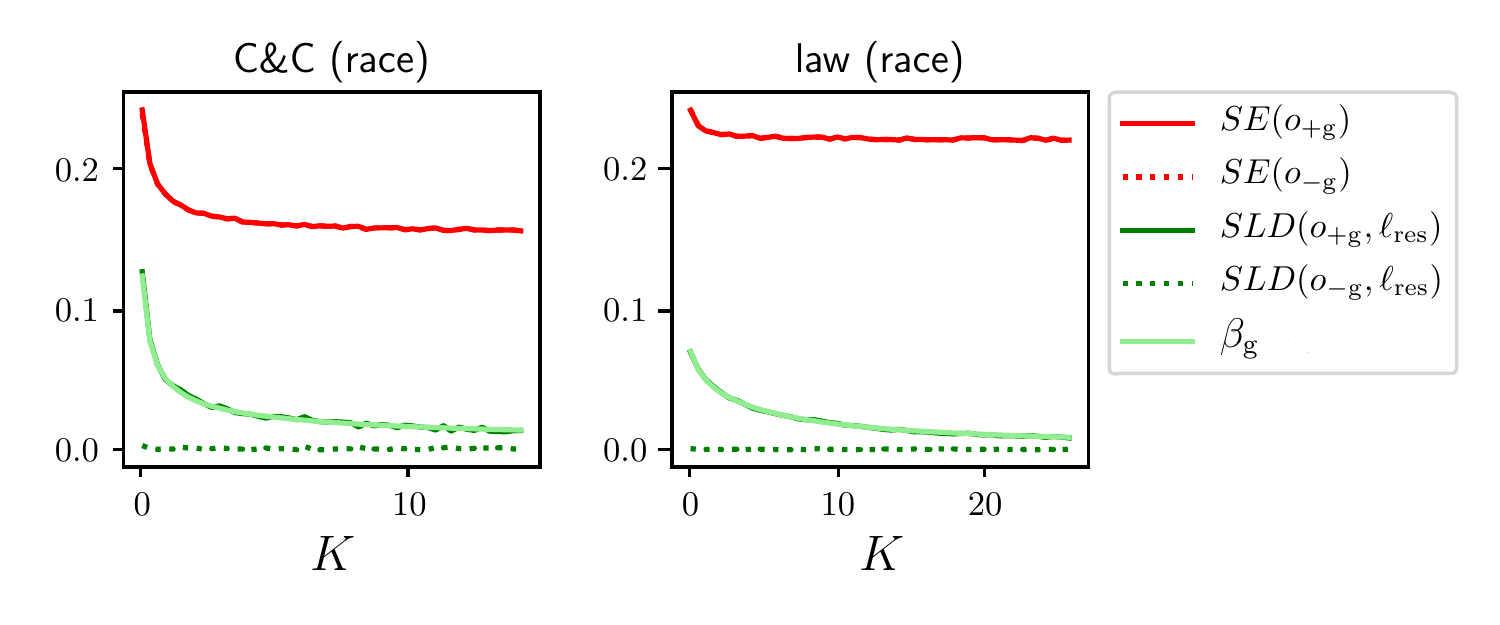}
	\caption{\label{fig:real_cov_shift}
	Loss discrepancy for the predictor learned on $\obnog$ (dotted green line) is always zero.
	Loss discrepancy for the predictor learned on $\obg$ (solid green line) converges to $0$ with rate of $\sO(\frac{1}{K})$.}
\end{figure}

We now study the persistence of loss discrepancy in the covariate shift setup introduced in \refsec{cov_shift}.
To simulate the shift, we consider the original distribution (uniform distribution over all data points), and a re-weighted distribution where weights are chosen  such that the mean of the features (except for the group membership) and the mean of the prediction target are the same between both groups. 
We compute such a re-weighting using linear programming; see Appendix~\ref{sec:experiment_details} for details.
For different values of $K$, we compute two least squares estimators (with and without group membership) on a batch of size $n=1000$ from the original distribution and a batch of size $Kn$ from the re-weighted distribution. 
We then calculate loss discrepancy and average squared error (SE) of both models on the re-weighted distribution. 

As stated in \refprop{cov_shift}, the estimator without group membership ($o_{-g}$) achieves zero residual loss discrepancy immediately (the dotted green line in \reffig{real_cov_shift}).
Meanwhile, the loss discrepancy of the estimator which uses group membership ($\obg$) vanishes more slowly with rate of $\sO(\frac{1}{K})$ (solid green line).

\pl{in general, try to be consistent with terminology throughout - sometimes you say learned predictor, sometimes training, models, sometimes least squares estimators}

	\section{Related Work and Discussion}
\label{sec:related_work}
While many papers focus on measuring loss discrepancy \citep{kusner2017,hardt2016,pierson2017fast,simoiu2017problem,khani2019mwld} and mitigating loss discrepancy \citep{calmon2017optimized,hardt2016,zafar2017fairness},
there are relatively few that study how loss discrepancy arises in machine learning models arises in the first place.

\citet{chen2018my} decompose the loss discrepancy into three components---bias, variance, and noise.
They mainly focus on the bias and variance, 
and also consider scenarios in which available features are not equally predictive for both groups. 
There are also lines of work which assume the loss discrepancy of the model is due to biased target values (e.g., \citet{madras2019fairness}).
Some work states that high loss discrepancy is due to lack of data for minority groups \citep{chouldechova2018frontiers}.
Some assume different groups have different functions (sometime in conflict with each other) \citep{dwork2018decoupled},
and therefore, fitting the same model for both groups is suboptimal.
In this work, we showed even when the prediction target is correct (not biased), with infinite data, the same function for both groups, equal noise for both groups, there is \emph{still} loss discrepancy.

Recently, there is some work showing that enforcing fairness constraints without accurately understanding how they change the predictor results in worse outcomes for both groups.
\citet{corbett2018measure} look at different group fairness notions and show how they can lead to worse results if groups have different risk distributions.
\citet{liu2018delayed} show that enforcing some fairness notions hurts the minorities in the long term.
\citet{lipton2018does} show that removing disparate treatment and disparate impact simultaneously causes in-group discrimination.
Our result that simple feature noise leads to loss discrepancy even under otherwise favorable conditions points at
a more fundamental problem: the lack of information about individuals.

	\paragraph{Problems with loss discrepancy notions.}
In this paper, we study statistical and counterfactual loss discrepancy,
which are well established in the literature.  However,
our results naturally hinge on the meaningfulness of these notions,
For completeness, we include a brief discussion about this.

Statistical loss discrepancy (\refdef{sld}) measures the loss discrepancy between groups.
SLD is valid if the considered distribution is representative (e.g., i.i.d samples) of the groups' real distribution. 
For example, consider the loan assignment task in which a model predicts whether a person will default or not.
If a model has low SLD in a loan dataset, there is no guarantee for low SLD if we use the model in the real-world.
Note that the loan dataset contains examples chosen by an entity to allocate loans (thus not i.i.d samples) and can be biased and not representative of the real-world distribution.
Thus, one can claim perfect accuracy for one group according to a dataset, while for that group 
only people who clearly will not default are in the dataset.
See \citet{corbett2018measure} for more examples.


Regarding counterfactual loss discrepancy (\refdef{cld}), three main concerns need to be considered.

\emph {Immutability of group identity}: 
Eminent researchers are opposed counterfactual reasoning with respect to an immutable characteristic (e.g., sex and race). 
They state that one cannot argue about the causal effect of a variable if its counterfactual cannot even be defined in principle \cite{holland1986statistics, freedman2004graphical}.
In response, some social scientists study the effect of some mutable variables associated with group membership (mainly associated with the perception of group membership).
For example, \citet{bertrand2004emily} studied the effect of ``racial soundingness'' of a name in a resume for getting an interview.
For more discussion on looking at race as a composite variable, see \citet{sen2016race}.

\emph {Post-treatment bias}: Characteristics such as race and sex are assigned at-conception before almost all other variables. 
Thus, considering the effect of these group identities while controlling for other variables that follow birth ($z$ in our setup) introduce post-treatment bias \cite{rosenbaum1984consequences} and can be misguided. 
Although this is a serious problem, CLD can still answer some valuable questions.
For example, according to the disparate treatment law,
a person is not liable for discrimination if she behaves in a trait-neutral manner.
In particular, an employer can make a decision based on characteristics that are crucial for job performance.
Informally, ``a decider should avoid its own discrimination not the ones that are already exists'' \cite{greiner2011causal}.
Note that CLD is asking about intentional discrimination through observation function (therefore, it is conditioned on $z$). 
CLD differs from SLD, which focuses on loss discrepancies between groups (without any conditioning on $z$), which might unintentionally occur. 
See \citet{greiner2011causal} for further discussion.

\emph {Inferring latent variables}: in real-world problems, we only observe $x$ and inferring $z$ from data is a hard (if not impossible) task \pl{cite}.
Inferring $z$ requires many strong assumptions regarding data generation.
In this work, we assumed we have access to the latent variable $z$ and focused on the effect of observation function on loss discrepancy. 
In particular, we showed that even the simple case that $x$ is a noisy version of $z$ leads to loss discrepancy.
There is a rich line of work on checking the fairness of a model when true features and observation function need to be inferred from data \citep{kusner2017,nabi2018fair,chiappa2019path,kilbertus2017avoiding,madras2019fairness}.


	\section{Conclusion and Future Work}
In this work, we first pointed out that in the presence of feature noise, the Bayes-optimal predictor of $y$ depends on the distribution of the inputs, which results in loss discrepancy for groups with different distributions.
For linear regression, we showed 
(i) feature noise causes high SLD, (ii) using group information mitigates SLD but increases CLD, and 
(iii) using group information also makes the loss discrepancy more persistent under covariate shift.
The studied loss discrepancies are not mitigated by collecting more data or designing a group-specific classifier, and designers should think of other methods such as feature replication to estimate the noise and de-noise the predictor \cite{carroll2006measurement}.

Our results rely on three main points: 
(i) we assume the true function is linear,
(ii) we study the predictor with minimum squared error among linear functions,
(iii) we consider two observation functions---feature noise with and without group information.
Relaxing these points, especially studying more complex observation functions,
would be a productive direction for future work.

	\paragraph{Reproducibility.} All code, data and experiments for this paper are available on the CodaLab platform at \url{https://worksheets.codalab.org/worksheets/0x7c3fb3bf981646c9bc11c538e881f37e}.

\paragraph{Acknowledgments.} This work was supported by Open Philanthropy Project Award.  
We would like to thank Emma Pierson, Pang Wei Koh, Ananya Kumar, and anonymous reviewers for useful feedback.

	\bibliography{refdb/all}
	\bibliographystyle{apalike}

\appendix
\onecolumn
\renewcommand{\thesection}{\Alph{section}}

\section{\refprop{obnog}}
\label{sec:obnog_proof}
\propObnog*

\begin{proof}
	Using least squares estimator, we have:
	\begin{align}
	\hat \beta =& \Sigma_x^{-1}\Sigma_{xy} \nonumber\\
	=& (\Sigma_z + \Sigma_u + \Sigma_{zu} + \Sigma_{uz})^{-1}(\Sigma_{zy} + \Sigma_{uy})
	\end{align}
	Due to assumptions ($u$ is independent of other variables), $\Sigma_{zu}=0$ and $\Sigma_{uy}=0$, and from \refeqn{best_parameters}, $\Sigma_{zy}=\Sigma_z \beta$.
		\begin{align}
		\hat \beta =&  (\Sigma_z + \Sigma_u)^{-1}\Sigma_z\beta \nonumber\\
		=& (I - (\Sigma_z + \Sigma_u)^{-1}\Sigma_u))\beta \nonumber \\
		=& (I - \Lambda)\beta
		\end{align}
	The intercept formulation is as follows:
	\begin{align}
	\hat \alpha = \beta^\top \E [z] + \alpha - \hat{\beta}^\top\E [z] = (\Lambda\beta)^\top \E [z] + \alpha,
	\end{align}
	Since $o(z, 0, u)=o(z,1,u)$, we have $\cld (\obnog, \lr ) = \cld (\obnog, \ls) =0$.
	
For computing SLD based on residual, we first compute the expected $\lr$ for the first group:
	\begin{align}
	\E [y - \hat y  \mid g=0] =& \E [\beta^\top z + \alpha - \hat \beta ^\top (z+u) - \hat \alpha \mid g=0]\\
	=&\E [(\Lambda\beta)^\top z - (\Lambda\beta)^\top \E [z] - \hat \beta ^\top u \mid g=0] 
	\end{align}
Using $\E [u]=0$ and $\E [z] = \pr [g=0] \E [z \mid g=0] + \pr [g=1] \E [z \mid g=1]$ we have:
	\begin{align}
\E [y - \hat y  \mid g=0]	=& (\Lambda\beta)^\top (\E[z \mid g=0] - \E [z]) \\ 
	=& (\Lambda\beta)^\top (\E[z \mid g=0] - \pr [g=0] \E [z \mid g=0] - \pr [g=1]\E [z\mid g=1] )\\ 
	=& (\Lambda\beta)^\top ((\pr [g=0]-1) (\E[z \mid g=1] - \E [z \mid g=0]))\\ 
	=&(\pr [g=0]-1)(\Lambda\beta)^\top\Delta\mu_z. \label{eqn:first_group_bias}
	\end{align}
	Note that the first group ($g=0$) have lower expected $\ell_r$ if its size is large or $\Delta\mu_z$ (projected on $\beta^\top\Lambda^\top$) is small.
	Similarly for the second group ($g=1$) we have:
	\begin{align}
	\E [y - \hat y  \mid g=1] =(1-\pr[ g=1])(\Lambda\beta)^\top\Delta\mu_z. \label{eqn:second_group_bias}
	\end{align}
	Computing the difference between the expected residuals of the groups we have: $\sld (\obnog, \lr )=|(\Lambda\beta)^\top\Delta\mu|$.

	For computing $\sld(\obnog, \ls)$, first note that the  squared error can be decomposed to squared of bias and variance,
	\begin{align}
	\E [(\hat y - y)^2] = \E [(\hat y - y)]^2 + \var [(\hat y - y)].
	\end{align}
	Using this decomposition, we have:
	\begin{align}
	\sld (\obnog, \ls ) =&\pab {\E [\ls \mid g=1] - \E [\ls \mid g=0]}\\
	=&\pab {\E [\lr \mid g=1]^2 - \E [\lr \mid g=0]^2 + \var [\lr \mid g=1] - \var [\lr \mid g=0]}
	\end{align}
	Using \refeqn{first_group_bias} and \refeqn{second_group_bias}, we have:
	\begin{align}
\E [\lr \mid g=1]^2 - \E [\lr \mid g=0]^2 =&  \pr [g=0]^2((\Lambda\beta)^\top \Delta\mu_z)^2 - \pr [g=1]^2 ((\Lambda\beta)^\top \Delta\mu_z)^2\\
=&-(\pr [g=1] - \pr [g=0]) ((\Lambda\beta)^\top \Delta\mu_z)^2. \label{eqn:sld1}
	\end{align}
For the difference between variances, we have:
		\begin{align}
\var [\lr \mid g=1] - \var [\lr \mid g=0] =&
		\var\pb{(\Lambda\beta)^\top z + \hat \beta^\top u \mid g=1} - \var\pb{(\Lambda\beta)^\top z + \hat \beta^\top u \mid g=0}\\
		\end{align}
	Since $u$ is independent of $z$ and $g$ the difference is only related to the difference between variances.
			\begin{align}
			\var [\lr \mid g=1] - \var [\lr \mid g=0] 
			=&\var [(\Lambda\beta)^\top z \mid g=1] - \var [(\Lambda\beta)^\top z \mid g=0] \\
			=&(\Lambda\beta)^\top \Delta\Sigma_z (\Lambda\beta) \label{eqn:sld2}
			\end{align}
Combining \refeqn{sld1} and \refeqn{sld2} completes the proof.

\end{proof}
\section{\refprop{obg}}
\label{sec:obg_proof}
\propObg*
\begin{proof}
	For simplicity, let $z'\eqdef z + u$.
	We are interested in finding the best linear estimator for $y$, given $z'$ and $g$.
	According to our assumptions we have:
	\begin{align}
	y &= \beta^\top z + \alpha\\
	\E [y \mid z',g] &= \beta^\top \E [z \mid z',g] + \alpha
	\end{align}
	
	First note that, we can represent $\E [z' \mid g]$ according to $g$ linearly as follows:
	\begin{align}
	\label{eqn:linear_g}
	\E [z' \mid g] = \E [z' \mid g=0] + (\E [z' \mid g=1] - \E [z' \mid g=0])g
	\end{align}
	We now write a linear predictor for $\E [z \mid z',g]$ given $z'$ and $g$ with some re-parametrization for simplicity. Define $v \eqdef z' - \E [z' \mid g]$ and $w \eqdef g - \E [g]$, we have:
	\begin{align}
	\gamma_0 + \gamma_1 (\underbrace{z' - \E [z'\mid g]}_{v}) + \gamma_2 (\underbrace{g - \E [g]}_{w})
	\end{align}
	If $n$ denotes the dimension of $z$ then $\gamma_0$ is $n\times 1$, $\gamma_1$ is $n\times n$ and finally $\gamma_2$ is $n\times 1$.
	
	First note that $\E [v] = \E [z'] - \E [\E [z' \mid g]] =0 $.
	Therefore, we have: 
	\begin{align}
\cov(v,w) =& \E [vw]\\
=&\E [z' g] - \E [z']\E [g] -\E [\E [z'\mid g]g] + \E [\E [z' \mid g]]\E [g]\\
=&\E [z' g] - \E [\E [z'\mid g]g] =0\\
	\end{align}
	 As a result, due to orthogonality, we can compute the parameters as follows:
	\begin{align}
	\gamma_1 =& \Sigma_v^{-1}\Sigma_{vz}
	\end{align}
	\begin{align} 
	\Sigma_v=\E \pb{\E\pb{(z'-\E [z'\mid g])(z'-\E [z' \mid g])^\top \mid g}} 
	=\E [\var [z' \mid g]] 
	\end{align}
For the covariance between $v$ and $z$ we have:
	\begin{align}
	\Sigma_{vz} =& \E [vz^\top] - \E [v]\E[z]^\top \nonumber\\ 
	=&\E\pb{(z'-\E [z'\mid g])z^\top} \nonumber \\
	=&\E [z'z^\top] - \E [\E [z'\mid g]z^\top]  \nonumber \\
	=&\E [zz^\top] + \E [uz^\top] - \E [ \E [z + u \mid g] z^\top] \nonumber \\
	=& \E [zz^\top] - \E[\E [z\mid g]z^\top] \nonumber \\
	=&\E \pb{\E \pb {zz^\top\mid g} - \E [z\mid g]\E [z\mid g]^\top} \nonumber \\
	=&\E [\var [z \mid g]]
	\end{align}
Combining the above equations and presenting $\hat \beta$ to two parts ($\hat \beta_z$ for coefficient of $z'$ and $\hat \beta_g$ for the coefficient of $g$), we have:
	\begin{align}
 \hat{\beta}_{z}=\gamma_1 = \E[\var (z' \mid g)]^{-1}\E [\var [z \mid g]].
	\end{align}

Using \refeqn{linear_g} and noting $\E [z' \mid g] = \E [z + u \mid g] = \E [z \mid g]$, we have $\gamma_2 = \Delta\mu$; therefore, $\hat\beta_g = \gamma_2 - \gamma_1\Delta\mu$
	
	Now define $\Lambda'$ to be:
	\begin{align}
	\Lambda' = I - (\E[\var (z' \mid g)]^{-1}\E [\var [z \mid g]])
	\end{align}
	Then the estimated parameters are:
	\begin{align}
	\hat \beta_z = (I-\Lambda')\beta, \quad \quad \hat{\beta}_g =  (\Lambda'\beta)^\top\Delta\mu, 
	\end{align}
	For the intercept we have:
	\begin{align}
	\hat \alpha - \alpha =& \beta^\top \E [z] - \hat\beta_z^\top \E [z] - \hat\beta_g \E [g] \nonumber \\
	=& ((I - (I-\Lambda'))\beta)^\top \E [z]- (\Lambda'\beta)^\top\Delta\mu_z\E [g] \nonumber\\
	=&(\Lambda'\beta)^\top (\E[z] - (\E [z \mid g=1] - \E [z \mid g=0])\E [g]) \nonumber\\
	=&(\Lambda'\beta)^\top\Big(\E [g]\E [z\mid g=1] + (1-\E [g])\E [z \mid g=0] - \E [g] \E [z \mid g=1]+\E[g]\E [z \mid g=0] \Big) \nonumber\\
	=&(\Lambda'\beta)^\top\E [z\mid g=0].
	\end{align}
	
	Now that we calculated the estimated parameters, we compute loss discrepancies.
	\sld{} based on residuals is zero, since $\E [\lr \mid g=0 ] = \E [\lr \mid g=1]=0$. We can calculate $\sld$ based on the squared error using the same techniques as \refprop{obnog}.
	\begin{align}
	\sld (\obg , \ls ) = (\Lambda'\beta)^\top \Delta\Sigma_z (\Lambda'\beta).
	\end{align}

We now compute $\cld (\obg, \lr )$. Recall $L_{g'} \eqdef \E [\ell (y, h(\ob (z,g,u)))\mid z]$.
\begin{align}
\cld (\obg, \lr ) =& \E \pb {\pab { L_1 - L_0}}\\
=&\E \pb {\pab {\E [y - \hat \beta_z^\top z - \hat \beta_z^\top u - \hat \beta_g - \hat \alpha \mid z] - \E [y - \hat \beta_z^\top z - \hat \beta_z^\top u  - \hat \alpha \mid z]}}\\
=&\pab {\beta_g} \\
=& \pab {(\Lambda'\beta)^\top \Delta\mu_z}.
\end{align}

For calculating $\cld (\obg, \ls)$, let $\mu_1 = \E [z \mid g=1]$ and $\mu_0 = \E [z \mid g=0]$, then we have:
\begin{align}
\cld (\obg, \ls ) =& \E \pb {\pab { L_1 - L_0}}\\
=&\E \pb {\pab {\E \pb {\p {y - \hat \beta_z^\top (z+u) - \hat \beta_g - \hat \alpha}^2 \mid z} - \E \pb {\p {y - \hat \beta_z^\top (z+u)  - \hat \alpha}^2 \mid z}}}\\
=& \E \pb {\pab {((\Lambda'\beta)^\top (z - \mu_1))^2 - ((\Lambda'\beta)^\top (z - \mu_0))^2}}\\
=&\E \pb {(\Lambda'\beta)^\top \p {(z-\mu_1)(z-\mu_1)^\top - (z-\mu_0)(z-\mu_0)^\top}    (\Lambda'\beta)}\\
=&\E \pb {(\Lambda'\beta)^\top \p {(\mu_1 - \mu_2) (2z - \mu_1 - \mu_2)^\top}    (\Lambda'\beta)} \\
=&|((\Lambda'\beta)^\top\Delta\mu_z )|2\E \pb {\pab {(\Lambda'\beta)^\top(z - \frac{\mu_1 +\mu_0}{2})}}
\end{align}

\end{proof}

\section{\refprop{cov_shift}}
\label{sec:cov_shift_proof}
\covShift*
	\begin{proof}
Due to assumption, $\Delta\mu_{z(\text{test})}=0$; therefore, $\cld (\obnog, \lr) = \sld (\obnog, \lr)=0$ and $\sld(\obg, \lr)$ and $\cld (\obg, \lr)$ are equal.


			As shown in \refprop{obg}, $\sld(\obg, \lr ) = (\Lambda'\beta)^\top\Delta\mu_{z(\text{train})}$.
			The mean difference $\Delta\mu_{z\text{(train)}} = 2t\mu$, and converges to zero as $t$ decreases.
			The challenge is to bound $\Lambda' = (\Sigma + 2t(1-t)\mu\mu^\top + \Sigma_u)^{-1}\Sigma_u$. 
			\cite{sherman1950adjustment} shows that if $A$ is nonsingular and $u,v$ are column vectors then
			\begin{align}
		(A+uv^\top)^{-1} =A^{-1} - \frac{1}{1+v^\top A^{-1}u}A^{-1}uv^\top A^{-1}
		\end{align}

		We can simplify $\Lambda'$ using the above equation. 
		
				\begin{align}
				((\Sigma + \Sigma_u)+2t(1-t)\mu\mu^\top)^{-1} =(\Sigma + \Sigma_u)^{-1} - \frac{2t(1-t)}{1+2t(1-t)\mu^\top (\Sigma + \Sigma_u)^{-1}\mu}(\Sigma + \Sigma_u)^{-1}\mu\mu^\top (\Sigma + \Sigma_u)^{-1}
				\end{align}
		
		First note that since we assumed the inverse of covariance matrix ($\Sigma + \Sigma_u$) exists, therefore, it should be positive definite.
		As a result $\mu^\top (\Sigma + \Sigma_u)^{-1} \mu \ge 0$.
		Also note that $0\le 2t(1-t) \le 1$; therefore, we can have the following bound:
		
		\begin{align}
		0\le\frac{2t(1-t)}{1+2t(1-t)\mu^\top A^{-1}\mu}  \le 1
		\end{align}
For simplicity, let $ r\eqdef \frac{2t(1-t)}{1+2t(1-t)\mu^\top A^{-1}\mu}$, using the above bound we can bound loss discrepancy when $\obg$ is used:
		\begin{align}
		\cld(\obg, \lr)=& (\Lambda'_{\text{train}}\beta)^\top \Delta\mu_{\text{train}}\\
	=&t \p { \p {(\Sigma+\Sigma_u)^{-1}\Sigma_u\beta}^\top(2\mu) - r \p{(\Sigma+\Sigma_u)^{-1}\mu\mu^\top (\Sigma+\Sigma_u)^{-1}\Sigma_u\beta}^\top(2\mu)} 
	\end{align}
	Defining $c_1 \eqdef \p {(\Sigma + \Sigma_u)^{-1}\Sigma_u\beta}^\top(2\mu) $ and $c_2 \eqdef \p{(\Sigma+\Sigma_u)^{-1}\mu\mu^\top (\Sigma+\Sigma_u)^{-1}\Sigma_u\beta}^\top(2\mu)$ we have:
	\begin{align}
		t \p{c_1 - |c_2| }\le \sld(\obg, \lr ) = \cld (\obg, \lr ) \le t \p {c_1+ |c_2|}
		\end{align}
	\end{proof}
	
\section{CLD and SLD are not comparable}

\begin{proposition}
	\label{prop:not_comparable}
	There exists a setting where $\cld = 0$ but $\sld \neq 0$,
	and another where $\sld = 0$ but $\cld \neq 0$.
\end{proposition}
\begin{proof}
	Assume there are only two kinds of individuals: $z_1$ and $z_2$.
	Assume $\frac{1}{4}$ of group one are $z_1$ and the rest are $z_2$.
	Assume $\frac{3}{4}$ of group two are $z_1$ and the rest are $z_2$.
	
	If a predictor has loss $1$ on $z_1$ and loss $2$ on $z_2$ (independent of their group membership), then $\cld=0$ but $\sld=0.5$.
	If the loss is the same as above when $g=0$ but when $g=1$ the loss is $2$ for $z_1$ and $1$ for $z_2$, then $\cld=1$, but $\sld=0$.
	Let $\ell$ be squared error, $z$ be a two dimensional vector, and $y=z_1 + z_2$. 
	Let $\pr (z_2 =0 \mid g=0) = 1$ but $\pr (z_2=0 \mid g=1) < 1$.
	If $\ob (z,g,u) = z_1$ then $\cld=0$; however, $\sld \neq 0$ and the loss for the second group is higher.
	Now if $\ob (z,g) = [z_1,g]$ then define $h(o(z,g)) = z_1 + \E [y - z_1 \mid g]$ therefore, $SLD=0$ while $\cld\neq 0$.
\end{proof}

\section{Infinite noise}
\label{sec:infinite_noise}

When noise is infinite, the predictor simply predicts $\mu_y= \E [y]$ for all data points when group membership is not available ($\obnog$), and it predicts the average of each group for the members of that group when group membership is available ($\obg$). 
In this case, statistical loss discrepancy is only related to the moments of groups on $y$, see \reftab{infinite_noise}.

\begin{table}[h]
	\centering
		\begin{tabular}{ll|l|l}
			\toprule
			& \multicolumn{1}{c|}{\bf \cld{}}& \multicolumn{1}{c|}{\bf \sld{}} & \multicolumn{1}{c}{\bf average performance}\\ \midrule
			\multirow{2}{*}{$\bf \boldsymbol{\lr}$} &$\obnog: 0$ & $\obnog: \Delta\mu_y$ & $\obnog: \E [\lr] =0$\\
			&$\obg: \Delta\mu_y$&$\obg: 0$ & $\obg: \E [\lr ]=0 $\\ \midrule
			\multirow{2}{*}{$\bf \boldsymbol{\ls}$} &$\obnog: 0$ & $\obnog: \Delta\sigma_y^2 + (1-2\E [g])\Delta\mu_y$ & $\obnog: \E [\ls ]= \sigma_{y}^2$ \\ 
			&$\obg: 2\Delta\mu_y\E [|y  -\mu_y-(\half - \E [g])\Delta\mu_y|]$& $\obg: \Delta\sigma_y^2$ & $\obg: \E [\ls ] =\sigma_{y\mid g}^2$\\\bottomrule
		\end{tabular}
		\caption{\label{tab:infinite_noise}A summery of metrics in the presence of infinite noise. Here, $\Delta\mu_y \eqdef \E [y \mid g=1] - \E [y \mid g=0]$ and $\Delta\sigma_y^2 \eqdef \var [y \mid g=1] - \var [y \mid g=0]$.}
	\end{table}
\section{General noise}

The independence assumption on the noise in \refsec{obnog} and \ref{sec:obg} enables us to have a closed-form for \cld{} and \sld{}.
Without any assumptions on the noise, we cannot specify anything about the estimated parameters more than \refeqn{best_parameters}.
However, we can still analyze the form of \sld{} given the estimated parameters

\label{sec:general_noise}
\begin{restatable}{proposition}{ildGldLr}
	\label{prop:general_noise}
	Fix $\ob(z,g,u)$ as an arbitrary observation function, such that for $x=\ob(z,g,u)$ the covariance matrix ($\Sigma_x$) is invertible.
	Let $\hat \beta$ and $\hat \alpha$ be the estimated parameters as computed in \refeqn{best_parameters}.
	Equalize dimensions of $\hat \beta$, $\beta$, $z$, and $x$  by adding extra $0$s at the end; and let $u = x -z$ denote the add-on error to the value of the true feature.
The statistical loss discrepancies are as follows:
	\begin{align}
	\sld (\ob, \lr )=&\Big |\hat \beta^\top \Delta\mu_x - \Delta\mu_y \Big | \label{eqn:first_sld}\\    
	=&\Big |(\hat\beta^\top - \beta^\top) \Delta\mu_z +  \hat \beta^\top \Delta\mu_u \Big |, \label{eqn:second_sld}
	\end{align}
	where for any random variable $t$, $\Delta\mu_t = \E [t \mid g=1] - \E [t \mid g=0]$.
	For SLD based on $\ls$, we have:
	\begin{align}
	\sld(\ob, \ls) =&\pab {\Delta\Sigma_y+ \hat{\beta}^\top\Delta\Sigma_x\hat{\beta} - 2 \hat{\beta}^\top\Delta\Sigma_{xy} }\\
	=&\pab {(\beta - \hat{\beta})^\top\Delta\Sigma_z(\beta - \hat \beta) + \hat{\beta}^\top\Delta\Sigma_u\hat \beta - 2( \beta - \hat\beta)^\top\Delta\Sigma_{zu}\hat{\beta}}
	\end{align}
		where for any two random variable $s,t$ we define $\Delta \Sigma_{st}$ to be:
		\begin{align}
		\Sigma_{st} \eqdef
		 \E [(s - \mu_s) (t - \mu_t)^\top \mid g =1 ] - \E [(s - \mu_s)^\top (t - \mu_t) \mid g =0 ].
		\end{align}

\end{restatable}

\begin{proof}
	
	\begin{align}
	\nonumber\sld (\ob, \lr ) =& \pab{\E [\lr \mid g=1] - \E [\lr \mid g=0]}\\
	\nonumber=& \Big |\E \pb {\beta^\top z + \alpha - \hat \beta^\top (z+u)  - \hat \alpha \mid g=1}-  \E \pb {\beta^\top z+ \alpha - \hat \beta^\top (z+u)  - \hat \alpha \mid g=0}\Big |\\
	\nonumber=&\Big |(\beta^\top - \hat\beta^\top) (\E [z \mid g=1] - \E [z \mid g=0]) - \hat \beta^\top (\E [u \mid g=1] - \E [u \mid g=0])\Big |\\
	=& \pab { (\beta - \hat \beta)^\top \Delta\mu_z - \hat \beta \Delta\mu_u}
	\end{align}
	
	$\sld (\ob, \lr )$ can also be formulated as follows:
	\begin{align}
	\nonumber\sld (\ob, \lr ) =& \pab{\E [\lr \mid g=1] - \E [\lr \mid g=0]}\\
	\nonumber=& \Big |\E \pb {y - \hat \beta^\top x  - \hat \alpha  \mid g=1}-  \E \pb {y - \hat \beta^\top x  - \hat \alpha  \mid g=0}\Big |\\
	=& \pab {\Delta \mu_y - \hat \beta^\top \Delta\mu_x} \label{eqn:other_form}
	\end{align} 
	This formulation implies that if $\Delta\mu_x =0$ and $\Delta\mu_y=0$ then for any arbitrary linear predictor $\sld (\ob, \lr )$ is always zero.

	For computing the statistical loss discrepancy based on squared error, first note that using \refeqn{best_parameters}, we have: $\hat \alpha = \beta \mu_z + \alpha - \hat \beta \mu_x \implies \alpha - \hat \alpha = - (\beta -\hat{\beta})^\top\mu_z + \hat{\beta}\mu_u$. We can now compute the expected squared error for the first group, as follows:
	\begin{align}
	\E [\ls \mid g=0] =& \E \Big[\p {\beta^\top z + \alpha - \hat \beta^\top (z+u)  - \hat \alpha}^2 \mid g=0\Big ]\\
	\nonumber=& \E\Big[\p {(\beta - \hat \beta)^\top (z - \mu_z) - \hat \beta^\top (u - \mu_u)}^2 \mid g=0\Big ]\\
	=&\E \Big[ (\beta - \hat \beta)^\top (z-\mu_z) (z-\mu_z)^\top (\beta - \hat \beta) + \hat{\beta}^\top(u-\mu_u)(u-\mu_u)^\top \hat \beta \\
	&- 2(\beta - \hat \beta)^\top(z-\mu_z)(u-\mu_u)^\top \hat \beta \mid g=0\Big] \nonumber\\
	\end{align}
	
	The expected squared error for $g=1$ is similar, computing the difference we have:
	\begin{align}
	\sld(\ob, \ls ) =& \pab {\E [\ls \mid g=1] - \E [\ls \mid g=0]}\\
	=&\pab {(\beta - \hat{\beta})^\top\Delta\Sigma_z(\beta - \hat \beta) + \hat{\beta}^\top\Delta\Sigma_u\hat \beta - 2( \beta - \hat\beta)^\top\Delta\Sigma_{zu}\hat{\beta}}
	\end{align}

	Same as \refeqn{other_form}, we can compute \sld{} based on squared error in another form as well:
	\begin{align}
	\E [\ls \mid g=0] =& \E \pb {\p {y - \hat \beta^\top x  - \hat \alpha}^2 \mid g=0}\\
	\nonumber=& \E \pb {\p {y - \hat \beta^\top x  - \mu_y + \hat \beta^\top \mu_x}^2 \mid g=0}\\
	\nonumber=& \E \pb {\p {(y - \mu_y) - \hat \beta^\top (x  - \mu_x)}^2 \mid g=0}
	\end{align}
	
	\begin{align}
	\sld(\ob, \ls) =& \pab {\E [\ls \mid g=1] - \E [\ls \mid g=0]}
	=\pab {\Delta\Sigma_y+ \hat{\beta}^\top\Delta\Sigma_x\hat{\beta} - 2 \hat{\beta}^\top\Delta\Sigma_{xy} }
	\end{align}

\end{proof}

%
Equation \refeqn{first_sld} states that if the average over features and the target values are the same between groups then $\sld (\obnog, \lr )= 0$ for any linear predictor.
Equation~\refeqn{second_sld} states that if we cannot estimate the true parameters ($\hat \beta \neq \beta$) then in order to have $\sld (\obnog, \lr )=0$ we should enforce different average add-on errors for the groups ($\Delta\mu_u \neq 0$).

\section{Datasets}
\label{sec:datasets}
{\bf Students\footnote{\tiny \url{https://archive.ics.uci.edu/ml/datasets/student+performance}}}\citep{cortez2008using}
This dataset represents student achievements in secondary education of one Portuguese school in mathematics subject.
The data features ($x$) include student grades, demographic, social, and school-related features  and it was collected by using school reports and questionnaires. 
The target ($y$) is the final year grade, and we set the groups ($g$) to be males and females.
Students dataset contains the students' 1st and 2nd period grades, which are strongly correlated with the prediction target (the final grade issued at the 3rd period).

{\bf Law School Admissions Council’s National Longitudinal Bar Passage Study \footnote {\tiny Downloaded from \url{https://github.com/jjgold012/lab-project-fairness} \citep {bechavod2017penalizing}}}\citep{wightman1998lsac}
This dataset consists of the records of graduate students in law major. 
 We set the target ($y$) to be the final Grade Point Average.
 The features include student grades and school-related features.  We consider two versions of this data, one where $g$ is race and the other where $g$ is sex.
 The Law school dataset contains features such as first-year GPA; which are strongly correlated with the prediction target.

{\bf Communities and Crime \footnote{\tiny\url{http://archive.ics.uci.edu/ml/datasets/communities+and+crime}}}\citep{redmond2002data}
This dataset represents communities within the United States, each data point represents a community, and the goal is to predict the per capital violent crimes ($y$) given features such as average income in that community. 
This dataset contains eight continuous features related to the race, specifying the number and percentage of different races within that community.
We replaced these features with a single binary feature; which is $1$ if a community is in top 50\% of communities with a majority of whites, and $0$ if otherwise.
We set this binary feature to be the indicate the groups.


\section{Experimental details for persistence of loss discrepancy}
\label{sec:experiment_details}
%
In the experiment section (\refsec{experiments}), we propose a re-weighted distribution under which groups have similar means.
Here we explain how to define a new distribution on data points such that according to the new distribution groups have same mean.
Let $X_1 \in \BR^{n_1\times d}$ be the data points for the first group and $X_2 \in \BR^{n_2\times d}$ be the data points for the second group.
We compute $p_1 \in \BR^{n_1}$ and $p_2 \in \BR^{n_2}$, with the following linear programming:

\begin{align}
\max\quad &\|p_1\|_1\\
s.t.\quad &p_1^\top X_1 = p_2^\top X_2\\
\quad & p_1^\top y_1 = p_2^\top y_2\\
& 0\le {p_1}_i\le 1, & i=1,\dots,n_1\\
& 0\le {p_2}_i\le 1, & i=1,\dots,n_2\\
&\|p_1\|_1 = \|p_2\|_1
\end{align}

We then sample points according to $\frac{p_1}{\|p_1\|_1}$ and  $\frac{p_2}{\|p_2\|_1}$.

\end{document}